\documentclass[11pt]{article}

\usepackage{fullpage,times,url}

\usepackage{amsthm,amsfonts,amsmath,amssymb,epsfig,color,float,graphicx,verbatim}
\usepackage{algorithmic}

\usepackage[ruled,vlined]{algorithm2e}
\usepackage{bbm}
\usepackage{authblk}

\usepackage{hyperref}
\hypersetup{
	colorlinks   = true, 
	urlcolor     = blue, 
	linkcolor    = blue, 
	citecolor   = black 
}

\newtheorem{theorem}{Theorem}[section]
\newtheorem{proposition}[theorem]{Proposition}
\newtheorem{lemma}[theorem]{Lemma}

\newtheorem{definition}[theorem]{Definition}

\newtheorem{remark}[theorem]{Remark}

\newcommand{\reals}{\mathbb{R}}

\newcommand{\sign}{\mathrm{sign}}

\usepackage[square,numbers]{natbib}

\newcommand{\norm}[1]{\|#1\|}
\newcommand{\inner}[1]{\langle#1\rangle}

\newcommand{\prob}[1]{{\mathbb{P}}\left[ #1 \right] }

\newcommand{\meand}[2]{{\mathbb{E}}_{#1}\left[ #2 \right] }
\newcommand{\abs}[1]{\left \lvert #1 \right \rvert}

\newcommand{\secref}[1]{Sec.~\ref{#1}}

\renewcommand{\eqref}[1]{Eq.~(\ref{#1})}
\newcommand{\lemref}[1]{Lemma~\ref{#1}}

\newcommand{\thmref}[1]{Thm.~\ref{#1}}
\newcommand{\propref}[1]{Proposition~\ref{#1}}
\newcommand{\appref}[1]{Appendix~\ref{#1}}

\makeatletter
\newcommand{\printfnsymbol}[1]{%
  \textsuperscript{\@fnsymbol{#1}}%
}
\makeatother

\title{Proving the Lottery Ticket Hypothesis: Pruning is All You Need}
\author[1]{Eran Malach\thanks{equal contribution}}
\author[2]{Gilad Yehudai\printfnsymbol{1}}
\author[1]{Shai Shalev-Shwartz}
\author[2]{Ohad Shamir}
\affil[1]{School of Computer Science, Hebrew University}
\affil[2]{Weizmann Institute of Science}
\date{}

\begin{document}
\maketitle

\begin{abstract}
    The lottery ticket hypothesis (Frankle and Carbin, 2018), states that a randomly-initialized network contains a small subnetwork such that, when trained in isolation, can compete with the performance of the original network. 
    We prove an even stronger hypothesis (as was also conjectured in Ramanujan et al., 2019), showing that for every bounded distribution and every target network with bounded weights, a sufficiently over-parameterized neural network with random weights contains a subnetwork with roughly the same accuracy as the target network, without any further training. 
\end{abstract}

\section{Introduction}
Neural network pruning is a popular method to reduce the size of a trained model, allowing efficient computation during inference time, with minimal loss in accuracy. However, such a method still requires the process of training an over-parameterized network, as training a pruned network from scratch seems to fail (see \cite{frankle2018lottery}). Recently, a work by \citet{frankle2018lottery} has presented a surprising phenomenon: pruned neural networks can be trained to achieve good performance, when resetting their weights to their initial values. Hence, the authors state \textit{the lottery ticket hypothesis}: a randomly-initialized neural network contains a subnetwork such that, when trained in isolation, can match the performance of the original network.

This observation has attracted great interest, with various follow-up works trying to understand this intriguing phenomenon. Specifically, very recent works by \citet{zhou2019deconstructing,ramanujan2019s} presented algorithms to find subnetworks that already achieve good performance, without any training. \cite{ramanujan2019s} stated the following conjecture: a sufficiently over-parameterized neural network with random initialization contains a subnetwork that achieves competitive accuracy (with respect to the large trained network), without any training. This conjecture can be viewed as a stronger version of the lottery ticket hypothesis.

In this work, we prove this stronger conjecture, in the case of over-parameterized neural networks. Moreover, we differentiate between two types of subnetworks: subnetworks where specific weights are removed (\textit{weight-subnetworks}) and subnetworks where entire neurons are removed (\textit{neuron-subnetworks}).
First, we show that a ReLU network of arbitrary depth $l$ can be approximated by finding a \textit{weight-subnetwork} of a random network of depth $2l$ and sufficient width.
Second, we show that depth-two (one hidden-layer) networks have \textit{neuron-subnetworks} that are competitive with the best random-features classifier (i.e. the best classifier achieved when training only the second layer of the network). Hence, we imply that for shallow networks, training the second layer of the network is equivalent to pruning entire neurons of a sufficiently large random network. 
In all our results, the size of initial network is polynomial in the problem parameters. In the case of the \textit{weight-subnetwork}, we show that the number of parameters in the pruned network is similar, up to a constant factor, to the number of parameters in the target network.

As far as we are aware, this is the first work that gives theoretical evidence to the existence of good subnetworks within a randomly initialized neural network (i.e., proving the strong lottery ticket hypothesis). Our results imply that fundamentally, pruning a randomly initialized network is as strong as optimizing the value of the weights. Hence, while the common method for finding a good network is to train its parameters, our work demonstrates that in fact, all you need is a good pruning mechanism. This gives a strong motivation to develop algorithms that focus on pruning the weights rather than optimizing their values.

\subsection{Related Work}

\paragraph{Neural Network Pruning}
Pruning neural networks is a popular method to compress large models, allowing them to run on devices with limited resources. Over the years, a variety of pruning methods were suggested, showing that neural network models can be reduced by up to 90\%, with minimal performance loss. These methods differ in two aspects: how to prune (the pruning criterion), and what to prune (specific weights vs. entire neurons or convolutional channels).
Works by \citet{lecun1990optimal, hassibi1993second,dong2017learning} explored the efficiency of network pruning based on second derivative conditions. Another popular method is pruning based on the magnitude of the weights \cite{han2015learning}. Other pruning techniques remove neurons with zero activation \cite{hu2016network}, or other measures of redundancy \cite{mariet2015diversity, srinivas2015data}.
While weight-based pruning achieves the best results in terms of network compression, the gain in terms of inference time is not optimal, as it cannot be efficiently utilized by modern hardware. To get an effective gain in performance, recent works suggested methods to prune entire neurons or convolutional channels \cite{yang2017designing, li2016pruning,molchanov2017variational,luo2017thinet}.

In our work, we show that surprisingly, pruning a random network achieves results that are competitive with optimizing the weights. Furthermore, we compare neuron-based pruning to weight-based pruning, and show that the latter can achieve strictly stronger performance. We are unaware of any theoretical work studying the power and limitation of such pruning methods.

\paragraph{Lottery Ticket Hypothesis}

In \cite{frankle2018lottery}, Frankle and Carbin stated the original \textbf{lottery ticket hypothesis}: \textit{A randomly-initialized, dense neural network contains a subnetwork that is initialized such that — when trained in isolation — it can match the test accuracy of the
original network after training for at most the same number of iterations}.
This conjecture, if it is true, has rather promising practical implications - it suggests that the inefficient process of training a large network is in fact unnecessary, as one only needs to find a good small subnetwork, and then train it separately. While finding a good subnetwork is not trivial, it might still be simpler than training a neural network with millions of parameters.

A follow up work by \citet{zhou2019deconstructing} claims that the ``winning-tickets'', i.e., the good initial subnetwork, already has better-than-random performance on the data, without any training. With this in mind, they suggest an algorithm to find a good subnetwork within a randomly initialized network that achieves good accuracy. Building upon this work, another work by \citet{ramanujan2019s} suggests an improved algorithm which finds an untrained subnetwork that approaches state-of-the-art performance, for various architectures and datasets. Following these observations, \cite{ramanujan2019s} suggested a complementary conjceture to the original lottery ticket hypothesis: \textit{within a sufficiently overparameterized neural network with
random weights (e.g. at initialization), there exists a subnetwork that achieves competitive accuracy}.

While these results raise very intriguing claims, they are all based on empirical observations alone. Our work aims to give theoretical evidence to these empirical results. We prove the latter conjecture, stated in \cite{ramanujan2019s}, in the case of deep and shallow neural networks. To the best of our knowledge, this is the first theoretical work aiming to explain the strong lottery ticket conjecture, as stated in \cite{ramanujan2019s}. 

\paragraph{Over-parameterization and random features} A popular recent line of works showed how gradient methods over highly over-parameterized neural networks can learn various target functions in polynomial time (e.g. \cite{allen2019learning},\cite{daniely2017sgd},\cite{arora2019fine},\cite{cao2019generalization}). However, recent works (e.g. \cite{yehudai2019power}, \cite{ghorbani2019linearized}, \cite{ghorbani2019limitations}) show the limitations of the analysis in the above approach, and compare the power of the analysis to that of random features. In particular, \cite{yehudai2019power} show that this approach cannot efficiently approximate a single ReLU neuron, even if the distribution is standard Gaussian. In this work we show that finding a shallow \emph{neuron-subnetwork} is equivalent to learning with random features, and that \emph{weight-subnetworks} is a \emph{strictly} stronger model in the sense that it can efficiently approximate ReLU neurons, under mild assumptions on the distribution (namely, that it is bounded).

\subsection{Notations}
We introduce some notations that will be used in the sequel. We denote by $\mathcal{X} = \{x \in \reals^d ~:~ \norm{x}_2 \le 1\}$ \footnote{The assumption that $\norm{x} \le 1$ is made for simplicity. It can be readily extended to $\norm{x}\leq r$ for any $r$ at the cost of having the network size depend polynomially on $r$.} our instance space.
For a distribution $\mathcal{D}$ over $\mathcal{X} \times \mathcal{Y}$, we denote the squared-loss of a hypothesis $h : \mathcal{X} \to \reals$ by:
\[
L_{\mathcal{D}}(h) = \meand{(x,y) \sim \mathcal{D}}{(h(x)-y)^2}~.
\]
For two matrices $A,B \in \reals^{m \times n}$, we denote by $A \odot B = [A_{i,j} B_{i,j}]_{i,j}$ the Hadamard (element-wise) product between $A$ and $B$. We use $U([-c,c]^k)$ to denote the uniform distribution on some cube around zero, and by $\mathcal{N}(0,\Sigma)$ a normal distribution with mean zero and covariance matrix $\Sigma$. For a matrix $H$ we denote by $\lambda_{\min}(H)$ its minimal eigenvalue.
For a matrix $A$, we denote by $\norm{A}_2$ the $L_2$ operator norm of $A$, namely $\norm{A}_2 := \lambda_{\max}(A)$ where $\lambda_{\max}$ is the largest singular value of $A$. We denote by $\norm{A}_{\max}$ the max norm of $A$, namely $\norm{A}_{\max} := \max_{i,j} \abs{A_{i,j}}$.

\section{Approximating ReLU Networks by Pruning Weights}
\label{sec:pruning_weights}

In this section we provide our main result, showing that a network of depth $l$ can be approximated by pruning a random network of depth $2l$. We show this for a setting where we are allowed to prune specific weights, and are not limited to removing entire neurons (i.e. finding \textit{weight-subnetworks}). Neuron-subnetworks are discussed in the next section. We further focus on networks with the ReLU activation, $\sigma(x) = \max \{x,0\}$.
We define a network $G: \reals^d \to \reals$ of depth $l$ and width\footnote{We consider all layers of the network to be of the same width for convenience of notations. Our results can easily be extended to cases where the size of the layers differ.} $n$ in the following way:
\[
G(x) = G^{(l)} \circ \dots \circ G^{(1)} (x)
\]
Where we have:
\begin{itemize}
    \item $G^{(1)}(x) = \sigma(W^{G(1)} x)$ for $W^{G(1)} \in  \reals^{d \times n}$.
    \item $G^{(i)}(x) =  \sigma(W^{G(i)} x)$ for $W^{G(i)} \in  \reals^{n \times n}$, for every $1 < i < l$.
    \item $G^{(l)}(x) = W^{G(l)} x$ for $W^{G(l)} \in \reals^{n \times 1}$
\end{itemize}

A \textbf{weight-subnetwork} $\widetilde{G}$ of $G$ is a network of width $n$ and depth $l$, with weights $W^{\widetilde{G}(i)} := B^{(i)} \odot W^{G(i)}$ for some mask $B^{(i)} \in \{0,1\}^{n_{in} \times n_{out}}$. Our main theorem in this section shows that for every target network of depth $l$ with bounded weights, a random network of depth $2l$ and polynomial width contains with high probability a subnetwork that approximates the target network:

\begin{theorem}
\label{thm:deep_network}
Fix some $\epsilon, \delta \in (0,1)$. Let $F$ be some target network of depth $l$ such that for every $i \in [l]$ we have $\norm{W^{F(i)}}_2 \le 1$,$\norm{W^{F(i)}}_{\max} \le \frac{1}{\sqrt{n_{in}}}$ (where $n_{in} = d$ for $i=1$ and $n_{in} = n$ for $i > 1$).
Let $G$ be a network of width $\mathrm{poly}(d,n,l, \frac{1}{\epsilon}, \log \frac{1}{\delta})$ and depth $2l$, where we initialize $W^{G(i)}$ from $U([-1, 1])$.
Then, w.p at least $1-\delta$ there exists a weight-subnetwork
$\widetilde{G}$ of $G$ such that:
\[
\sup_{x \in \mathcal{X}}\abs{\widetilde{G}(x) - F(x)} \le \epsilon
\]
Furthermore, the number of active (non-zero) weights in $\widetilde{G}$ is $O(dn +n^2l)$.
\end{theorem}

\begin{remark}
We note that the initialization scheme of the network considered in \thmref{thm:deep_network} is not standard Xavier initialization. The reason is that in standard Xavier initialization the weights are normalized such that the gradient's variance at initialization will not depend on the network's size. Here we don't calculate the gradient but only prune some of the neurons. Thus, the magnitude of the weights does not depend on the width of the network.
That said, the theorem can be easily extended to any initialization which is a uniform distribution on some interval around zero, by correctly scaling the network's output.
\end{remark}

 Since the number of parameters in the function $F$ is $dn+n^2(l-2)+n$, the above shows that the number of active weights in the pruned network is similar, up to a constant factor, to the number of parameters in $F$.
 Note that the width of the random network has polynomial dependence on the input dimension $d$, the width of the target network $n$ and its depth $l$. While the dependence on the width and depth of the target network is unavoidable, the dependence on the input dimension may seem to somewhat weaken the result. Since neural networks are often used on high-dimensional inputs, such dependence on the input dimension might make our result problematic for practical settings in the high-dimension regimes. However, we note that such dependence could be avoided, when making some additional assumptions on the target network. Specifically, if we assume that the target network has sparse weights in the first layer, i.e. - each neuron in the first layer has at most $s$ non-zero weights, then we get dependence on the sparsity $s$, rather than on the input dimension $d$. This is shown formally in the appendix.

The full proof of \thmref{thm:deep_network} can be found in \appref{append:proofs from pruning weights}, and here we give a sketch of the main arguments.
The basic building block of the proof is showing how to approximate a single ReLU neuron of the form $x \mapsto \sigma(\inner{w^*,x})$ by a two layer network. Using the equality $a = \sigma(a) - \sigma(-a)$, we can write the neuron as:
\begin{equation} \label{eqn:single_target}
x \mapsto \sigma\left(\sum_{i=1}^d w^*_i x_i\right) = 
\sigma\left(\sum_{i=1}^d \sigma(w^*_i x_i) - \sum_{i=1}^d\sigma(-w^*_i x_i)\right)
\end{equation}
Now, consider a two layer network of width $k$ and a single output neuron, with a pruning matrix $B$ for the first layer. It can be written as
$
x \mapsto \sigma\left(\sum_{j=1}^k u_j \sigma\left(\sum_{t=1}^d B_{j,t} W_{j,t} x_t\right)\right).
$
Suppose we pick, for every $i$, two indexes $j_1(i), j_2(i)$, and set the matrix $B$ s.t. $B_{j_1(i),i}, B_{j_2(i),i} = 1$ and all the rest of the elements of $B$ are zero. It follows that the pruned network can be rewritten as
\begin{equation} \label{eqn:pruning_for_single_target}
\begin{split}
x &\mapsto \sigma\left(\sum_{i=1}^d u_{j_1(i)} \sigma(W_{j_1(i),i} x_i) + \sum_{i=1}^d u_{j_2(i)} \sigma(W_{j_2(i),i} x_i) \right) \\
&= \sigma\left(\sum_{i=1}^d \sign(u_{j_1(i)}) \sigma(|u_{j_1(i)}|\,W_{j_1(i),i} x_i) + \sum_{i=1}^d \sign(u_{j_2(i)}) \sigma(|u_{j_2(i)}|\,W_{j_2(i),i} x_i) \right)  
\end{split}
\end{equation}
Comparing the right-hand sides of Equations \ref{eqn:single_target} and \ref{eqn:pruning_for_single_target}, we observe that they will be at most $\epsilon$ away from each other provided that for every $i$, $\sign(u_{j_1(i)}) \ne \sign(u_{j_2(i)})$, $\abs{\abs{u_{j_1(i)}}\,W_{j_1(i),i} - \sign(u_{j_1(i)})w^*_i} \le \epsilon/{2d}$ and $\abs{\abs{u_{j_2(i)}}\,W_{j_2(i),i} - \sign(u_{j_2(i)})w^*_i} \le \epsilon/{2d}$. Finally, fixing $i$ and picking $u_{j_1(i)}, u_{j_2(i)},W_{j_1(i),i}, W_{j_2(i),i}$ at random, the requirements would be fulfilled with probability of $\Omega(\epsilon/d)$. Hence, if $k \gg d/\epsilon$, for every $i$ we would be able to find an appropriate $j_1(i),j_2(i)$ with high probability. 
Note that the number of weights we are actually using is $2d$, which is only factor of $2$ larger than the number of original weights required to express a single neuron.

The construction above can be easily extended to show how a depth two ReLU network can approximate a single ReLU layer (simply apply the construction for every neuron). By stacking the approximations, we obtain an approximation of a full network. Since every layer in the original network requires two layers in the newly constructed pruned network, we require a twice deeper network than the original one. 


We can derive a slightly stronger result in the case where the target network is a depth-two network. Specifically, we can show that a depth-two network can be approximated by pruning a depth-three random network (rather than pruning a depth-four network, as implied from \thmref{thm:deep_network}):


\begin{theorem}
\label{thm:weight_pruning}
Fix some target two-layer neural network $F$ of width $n$, and fix $\epsilon, \delta \in (0,1)$.
Let $G$ be a random three-layer neural network of width $\mathrm{poly}\left(d,n,\frac{1}{\epsilon}, \log\left(\frac{1}{\delta}\right)\right)$, with weights initialized from $U([-1,1])$. Then, with probability at least $1-\delta$, there exists a weight-subnetwork $\widetilde{G}$ of $G$, such that:
\[
\sup_{x \in \mathcal{X}}\abs{F(x) - \widetilde{G}(x)} \le \epsilon
\]
Furthermore, the number of active (non-zero) weights in $\widetilde{G}$ is $O(dn)$.
\end{theorem}

\subsection{Universality and Computational Efficiency of Pruning}

We showed that in terms of expressive power, pruning weights in a randomly initialized over-parameterized network can approximate a target ReLU network of any depth. Using well-known results in the literature of neural networks, this result implies two interesting corollaries:

\paragraph{Universal Approximation Using Weight Pruning} It has been long known that neural networks are universal approximators: they are able to approximate an arbitrary function up to arbitrary accuracy (for example, see \cite{stinchcombe1989universal, scarselli1998universal}). Since we show that pruning a network with random weights can approximate any target network, this implies that pruning a random network is also a universal approximation scheme.

\paragraph{Pruning Weights is Computationally Hard} It is well known in the literature of neural networks that learning even a depth-two ReLU network is computationally hard in the general case (see \cite{livni2014computational, manurangsi2018computational, boob2018complexity}). From these results, it is immediate that weight-pruning of random ReLU networks, deep or shallow, is computationally hard as well.
Indeed, if we had an efficient algorithm that finds an optimal weight-subnetwork of a three-layer network, from \thmref{thm:weight_pruning} this algorithm approximates the best depth-two network (for some fixed width). But in general, approximating the best depth-two network on an arbitrary distribution is computationally hard (under certain hardness assumptions), which leads to a contradiction.
So, there is no efficient algorithm that is guaranteed to return an optimal weight-subnetwork for any input distribution.

\section{Equivalence Between Pruning Neurons and Random Features}\label{sec:equivalence}
In this section we analyze the power of pruning entire neurons in a depth-two network. 
The main question we are interested in is the following: suppose that a function $f$ can be well approximated by a depth-two network $g$ of polynomial width (in the relevant parameters). Is the function $f$ also well approximated by pruning entire neurons of a randomly initialized depth-two network of a polynomial width?
Here we show that the answer is negative, and in fact pruning entire neurons is equivalent to the well known random features model (e.g. \cite{rahimi2008random}, \cite{rahimi2008uniform}). Intuitively, we show that whenever training only the last layer of the network suffices, it is also possible to construct a good sub-network by pruning entire neurons.

Formally, consider a width $k$ two-layer neural network defined by $g:\mathbb{R}^d \rightarrow \mathbb{R}$ as follows:
\[
g(x) = u^\top \sigma(Wx) = \sum_{i=1}^k u_i\sigma(\inner{w_i,x})
\]
where $u_i$ is the $i$-th coordinate of $u$ and $w_i$ is the $i$-th row of $W$. A network $\tilde{g}$ is a \textbf{neuron-subnetwork} of $g$ if there exists a vector $b\in\{0,1\}^k$ such that:
\[
\tilde{g}(x) = (u\odot b)^\top \sigma(Wx) = \sum_{i=1}^k (u_i \cdot b_i) \sigma(\inner{w_i,x}).
\]
So, $\tilde{g}$ is also a $2$-layer neural network, which contains a subset of the neuron of $g$. Next, we define the random features model:

\begin{definition}
Suppose we sample $w_1,\dots,w_k \sim D$ from some distribution $D$, a \textbf{random features model} over $w_1,\dots,w_k$ and activation $\sigma$ is any function of the form:
\[
f(x) = \sum_{i=1}^k u_i \sigma(\inner{w_i,x})
\]
for $u_1,\dots,u_k\in\mathbb{R}$.
\end{definition}

Training a $2$-layer random features model is done by training only the second layer, i.e. training only the weights $u_1,\dots,u_k$. This is equivalent to training a linear model over the features $\sigma(\inner{w_i,x})$, which are chosen randomly. We show that neuron-subnetworks are competitive with random features:

\begin{theorem}\label{thm:equivalence nn rf}
Let $D$ be any distribution over $\mathcal{X} \times [-1,+1]$, and let $\sigma:\mathbb{R}\rightarrow\mathbb{R}$ be $L$-Lipschitz with $\sigma(0)\leq L$. Let $\epsilon,\delta > 0,~n\in \mathbb{N}$ and $D^*$ a distribution over $\{w:\norm{w}\leq 1\}$ such that for $w_1,\dots,w_n \sim D^*$ w.p $> 1-\delta$ there exist $u_1,\dots,u_n\in\mathbb{R}$ such that $|u_i| \leq C$ and the function $f(x) = \sum_{i=1}^n u_i\sigma(\inner{w_i,x})$ satisfies that $L_D(f) \leq \epsilon$. Let $k\geq poly\left(C,n,L,\frac{1}{\epsilon}, \frac{1}{\delta}\right)$, and suppose we initialize a $2$-layer neural network $g$ with width $k$ where $w_i\sim D^* $, and $u_i\sim U([-1,1])$. Then there exists a neuron-subnetwork $\tilde{g}$ of $g$ and constant $c>0$ such that $L_D(c\tilde{g}) \leq \epsilon$.
\end{theorem}

The full proof can be found in \appref{appen:proofs from sec equivalence}. \thmref{thm:equivalence nn rf} shows that for any distribution over the data, if a random features model can achieve small loss, then it is also possible to find a neuron-subnetwork of a randomly initialized network (with enough width) that achieves the same loss. This means that pruning neurons is competitive with the random features model. On the other hand, if for some distribution over the data it is possible to find a neuron-subnetwork of a randomly initialized network that achieves small loss, then clearly it is possible to find a random features model that achieves the same loss. Indeed, we can set the weights of the random features model to be the same as in the neuron-subnetwork, where pruned weights are equal to zero.

To summarize, \thmref{thm:equivalence nn rf} and the argument above shows an equivalence between random features and neuron-subnetworks: For a distribution $D$, there is a random features model $f$ with $k$ features such that $L_D(f) \leq \epsilon$ if-and-only-if for a randomly initialized network with width polynomial in $k, \frac{1}{\epsilon}$ and $\frac{1}{\delta}$, w.p $>1-\delta$ there exists a neuron-subnetwork $\tilde{g}$ such that $L_D(\tilde{g}) \leq \epsilon$.

A few recent works (e.g. \cite{yehudai2019power}, \cite{ghorbani2019linearized}, \cite{ghorbani2019limitations}) studied the limitations of random features. In particular, \cite{yehudai2019power} show that a random features model cannot approximate a single ReLU neuron even under standard Gaussian distribution, unless the amount of features or the magnitude of the weights (or both) are exponential in the input dimension. Thus, the above equivalence also shows a limitation of neuron-subnetworks - they cannot efficiently approximate a single ReLU neuron, just as random features can't. This means that the weight-subnetwork model shown in \secref{sec:pruning_weights} is significantly stronger than the neuron-subnetwork model.

The intuition behind the proof of \thmref{thm:equivalence nn rf} is the following: Assume we initialize a $2$-layer neural network of width $n=k\cdot m$ where $k$ is as in the theorem, and $m$ is some large number (that depends on $\frac{1}{\epsilon}, \frac{1}{\delta}$). We think of it as initializing $m$ different networks of width $k$, and from the assumption, for most of these networks there exists a random features model that achieves small loss. For each of these networks we prune a neuron if its randomly initialized weight in the second layer is far from its corresponding random features model's weight. Note that since we initialize the weights i.i.d., then we prune each neuron with the same probability and independently of the other neurons. To finish the proof, we use a concentration of measure argument to show that averaging many such pruned networks competes with the random features model, and thus also achieves small loss on the input distribution.


\subsection{Learning Finite Datasets and RKHS Functions via Neuron-Subnetworks}\label{sec:pruning neurons}
In this subsection we show that pruning entire neurons may prove beneficial, despite the inherent limitations discussed previously. We focus on two popular families of problems, which are known to be solvable by training depth-two networks:
\begin{enumerate}
    \item Overfitting a finite sample: $S = \left\{(x_1,y_1),\dots,(x_m,y_m) \in \mathcal{X} \times [-1,1]\right\}$. This is equivalent to finding a neuron-subnetwork which minimizes the empirical risk on the sample $S$. This setting is considered in various recent works (for example in \cite{du2018gradient2}, \cite{du2018gradient}, \cite{allen2018convergence}).
    \item Learning RKHS: given an activation function $\sigma:\mathbb{R}\rightarrow\mathbb{R}$ we consider a target function from the following set of functions:
    \[
    \mathcal{F}_C = \left\{ f(x) = c_d\int_{w\in \left[-\frac{1}{\sqrt{d}}, \frac{1}{\sqrt{d}} \right]^d} h(w)\sigma(\inner{w,x})dw: \sup_w |h(w)| \leq C \right\}
    \]
    where $c_d = \left(\frac{\sqrt{d}}{2}\right)^d$ is a normalization term. The set $\mathcal{F}_\infty$ is actually the RKHS of the kernel $K(x,y) = \mathbb{E}_{w\in U\left(\left[-\frac{1}{\sqrt{d}}, \frac{1}{\sqrt{d}} \right]^d\right)}\left[\sigma(\inner{w,x})\cdot\sigma({w,y})\right]$. In particular, for $\sigma$ which is not a polynomial, the set $\mathcal{F}_\infty$ contains all continuous functions (see \cite{leshno1993multilayer}). This setting is considered in \cite{cao2019generalization}, \cite{sun2018random}.
\end{enumerate}

The main theorem of this section is the following:
\begin{theorem}\label{thm:neruon subnetwrok two variations}
Let $\epsilon,\delta >0$ and let $\sigma:\mathbb{R}\rightarrow\mathbb{R}$ be $L$-Lipschitz with $\sigma(0)\leq L$.
Let $g$ be a randomly initialized $2$-layer neural network of width $k$ such that $w_i\sim U\left(\left[-\frac{1}{\sqrt{d}}, \frac{1}{\sqrt{d}} \right]^d\right) $, and $u_i\sim U([-1,1])$.
\begin{enumerate}
    \item (Finite dataset) Let $S = \left\{(x_1,y_1),\dots,(x_m,y_m)\in\mathcal{X}\times [-1,+1]\right\}$. Let $H$ be the $m\times m$ matrix defined by $H_{i,j} = \mathbb{E}_w\left[ \sigma(\inner{w,x_i})\sigma(\inner{w,x_j})\right]$ and assume that $\lambda_{\min}(H) = \lambda > 0$. If $k \geq poly\left(m, \frac{1}{\lambda}, L, log\left(\frac{1}{\delta}\right), \frac{1}{\epsilon}\right)$ then w.p $> 1- \delta$ there exists a neuron-subnetwork  $\tilde{g}$ and a constant $c > 0$ such that:
    \[
    \sup_{i=1,\dots,m} \abs{c\tilde{g}(x_i) - y_i} \leq \epsilon
    \]
    \item (RKHS function)  Let $f \in \mathcal{F}_C$. If $k \geq poly\left(C, L, log\left(\frac{1}{\delta}\right), \frac{1}{\epsilon}\right)$
    then w.p $>1-\delta$ there exists a neuron-subnetwork $\tilde{g}$ and a constant $c>0$ such that:
    \[
    \sup_{x \in \mathcal{X}} \abs{c\tilde{g}(x) - f(x)} \leq \epsilon
    \]
\end{enumerate}
\end{theorem}

\begin{remark}
For the finite dataset case, the assumption on the minimal eigenvalue $\lambda$ of the matrix $H$ is standard and assumed in similar forms in other works which approximate a finite dataset using random features approach (see \cite{du2018gradient2}, \cite{du2018gradient}, \cite{panigrahi2019effect}). 
\end{remark}

In both versions of the theorem, the network's width does not depend on the dimension of the input data. It does depend on the ``complexity'' of the target distribution. In the finite dataset case the network's width depends on the number of examples $m$ and on the value of $\frac{1}{\lambda}$. In the RKHS function case, it depends on the constant $C$ which defines the size of the function class $\mathcal{F}_C$ from which the target function is taken. 

Note that in a binary classification task (where that labels are $\pm1$) over a finite dataset, \thmref{thm:neruon subnetwrok two variations} shows that we can achieve zero loss (with respect to the $0-1$ loss), even if we don't scale $\tilde{g}(x)$ by a constant $c$. To show this, we use \thmref{thm:neruon subnetwrok two variations} with $\epsilon = 1/2$ to get that for every pair $(x,y)$ in the finite dataset we have $|c\tilde{g}(x)-y| \leq 1/2$, since $c>0$ and $y\in\{1,-1\}$ we get that $\text{sign}(\tilde{g}(x)) = \text{sign}(y)$.

We give a short proof intuition for \thmref{thm:neruon subnetwrok two variations}, the full proof is in appendix \ref{apen:proofs of pruning neurons}. We initialize a $2$-layer neural network of width $k=k_1\cdot k_2$, this can be thought as initializing $k_2$ different networks, each of width $k_1$. The idea is to choose $k_1$ large enough so that w.h.p. a random features model with $k_1$ features would be able to approximate the target (either finite dataset or RKHS function). Next, for each network of size $k_1$ we prune a neuron if it is far from its corresponding random features model. We finish by using a concentration of measure argument to conclude that averaging over $k_2$ such networks (for a large enough $k_2$) yields a good approximation of the target.

\begin{remark}
The proof of \thmref{thm:neruon subnetwrok two variations} actually provides an algorithm for pruning $2$-layer neural networks:
\begin{itemize}
    \item Randomly initialize a $2$-layer neural network of width $k = k_1\cdot k_2$.
    \item For each subnetwork of width $k_1$ - optimize a linear predictor over the random weights from the first layer.
    \item Let $\epsilon$ be a confidence parameter, prune each neuron if its distance from the corresponding weight of the trained linear predictor is more than $\epsilon$.
\end{itemize}
This algorithm runs in polynomial time, but it is obviously very naive. However, it does demonstrate that there exists a polynomial time algorithm for pruning neurons in shallow networks. We leave a study of more efficient algorithms for future work.
\end{remark}

\section{Discussion/Future Work}

We have shown strong positive results on the expressive power of pruned random networks. However, as we mentioned previously, our results imply that there is no efficient algorithm for weight-pruning of a random network, by reduction from hardness results on learning neural networks. Hence, weight-pruning is similar to weight-optimization in the following sense: in both methods there exists a good solution, but finding it is computationally hard in the worst case. That said, similarly to weight optimization, heuristic algorithms for pruning might work well in practice, as shown in \cite{zhou2019deconstructing, ramanujan2019s}. Furthermore, pruning algorithms may enjoy some advantages over standard weight-optimization algorithms. First, while weight-optimization requires training very large networks and results in large models and inefficient inference, weight-pruning by design achieves networks with preferable inference-time performance. Second, weight-optimization is largely done with gradient-based algorithms, which have been shown to be suboptimal in various cases (see \cite{shalev2017failures, shamir2018distribution}). Pruning algorithms, on the other hand, can possibly rely on very different algorithmic techniques, that might avoid the pitfalls of gradient-descent.

To conclude, in this work we showed some initial motivations for studying algorithms for pruning random networks, which we believe set the ground for numerous future directions. An immediate future research direction is to come up with a heuristic pruning algorithm that works well in practice, and provide provable guarantees under mild distributional assumptions.
Other interesting questions for future research include understanding to what extent the polynomial dependencies of the size of the neural network before pruning can be improved, and generalizing the results to other architectures such as convolutional layers and ResNets.


\paragraph{Acknowledgements:} This research is supported by the European Research Council (TheoryDL project), and by European Research Council (ERC) grant 754705.

\newpage

\setcitestyle{numbers}
\bibliographystyle{abbrvnat}
\bibliography{mybib}

\newpage

\appendix

\section{Proofs of Section \ref{sec:pruning_weights}}\label{append:proofs from pruning weights}
We prove the theorem in a general manner, where we assume that each vector $w^*$ is $s$-sparse, that is, it only has $s$ non-zero coordinates. To prove \thmref{thm:weight_pruning} we assign $s=d$.

We start by showing that the function $x \mapsto \alpha x_i$ can be approximated by prunning a two-layer network:

\begin{lemma}
\label{lem:one_coord}
Let $s \in [d]$, and 
fix some scalar $\alpha \in [-\frac{1}{\sqrt{s}}, \frac{1}{\sqrt{s}}]$, index $i \in [d]$, and some $\epsilon, \delta > 0$.
Let $w^{(1)}, \dots, w^{(k)} \in \reals^d$ chosen randomly from $U([-1, 1]^d)$, and $u^{(1)}, \dots, u^{(k)} \in [-1,1]$ chosen randomly from $U([-1,1])$. Then, for $k \ge \frac{4}{\epsilon^2} \log(\frac{2}{\delta})$, w.p at least $1-\delta$ there exists a binary mask $b^{(1)}, \dots, b^{(k)} \in \{0,1\}^d$, such that
$g(x) = \sum_j u^{(j)} \sigma(\inner{w^{(j)} \odot b^{(j)}, x})$ satisfies
$|g(x) - \alpha x_i| \le 2\epsilon$, for $\norm{x}_\infty \le 1$. Furthermore, we have $\sum_j \norm{b^{(j)}}_0 \le 2$ and $\max_j \norm{b^{(j)}}_0 \le 1$.
\end{lemma}

\begin{proof}
If $|\alpha| \le \epsilon$ then choosing $b^{(1)} = \dots = b^{(k)} = (0, \dots, 0)$ gives the required. Assume $|\alpha| \ge \epsilon$, and assume w.l.o.g that $\alpha > 0$. Fix some $j \in [k']$. Note that:
\[
\prob{|w^{(j)}_i-\alpha| \le \epsilon \wedge |u^{(j)}-1| \le \epsilon}
= \prob{|w^{(j)}_i-\alpha| \le \epsilon}\prob{|u^{(j)}-1| \le \epsilon}
= \frac{\epsilon}{2} \cdot \frac{\epsilon}{2} = \frac{\epsilon^2}{4}~,
\]
and similarly $\prob{|w^{(j)}_i+\alpha| \le \epsilon \wedge |u^{(j)}+1| \le \epsilon} \le \frac{\epsilon^2}{4}$. Therefore, we have:
\[
\prob{\nexists j \in [k] ~s.t~ |w_i^{(j)}-\alpha| \le \epsilon \wedge |u^{(j)}-1| \le \epsilon} = \left(1-\frac{\epsilon^2}{4}\right)^{k} \le \exp \left(-\frac{k\epsilon^2}{4}\right) \le \frac{\delta}{2}~,
\]
where we used the assumption that $k  \ge \frac{4}{\epsilon^2} \log(\frac{2}{\delta})$, and similarly:
\[
\prob{\nexists j \in [k'] ~s.t~ |w_i^{(j)}+\alpha| \le \epsilon \wedge |u^{(j)} + 1| \le \epsilon} \le \frac{\delta}{2}~.
\]
Therefore, using the union bound, w.p at least $1-\delta$ there exist $j,j'$ such that $|w_i^{(j)}-\alpha| \le \epsilon, |u^{(j)}-1| \le \epsilon$ and $|w_i^{(j')}+\alpha| \le \epsilon, |u^{(j')} + 1| \le \epsilon$ and since $|\alpha| \ge \epsilon$ we get $j \ne j'$.
Now, setting $b_i^{(j)} = 1, b^{(j')}_i = 1$, and the rest to zero, we get that:
\[
g(x) = u^{(j)} \sigma(w^{(j)}_i x_i) + u^{(j')} \sigma(w^{(j')}_i x_i)
\]
We will use the fact that $\sigma(a) - \sigma(-a) =a $ for every $a\in\mathbb{R}$. If $x_i \ge 0$, we get that $g(x) = u^{(j)} w^{(j)}_i x_i$ and therefore:
\[
|g(x)-\alpha x_i| =  |x_i||u^{(j)} w^{(j)}_i - \alpha| \le 
|u^{(j)} w^{(j)}_i -u_j \alpha| + |u^{(j)} \alpha - \alpha|
\le |u^{(j)}||w_i^{(j)} - \alpha| + |u^{(j)}-1||\alpha| \le 2\epsilon
\]
In a similar fashion, we get that for $x_i < 0$ we have
 $|g(x)- \alpha x_i| = |x_i||u^{(j')} w_i^{(j')} - \alpha| \le 2 \epsilon$,
 which gives the required. Since we have $\norm{b^{(j)}}_0 = 1$, $\norm{b^{(j')}}_0 = 1$ and $\norm{b^{(j'')}}_0 = 0$ for every $j'' \ne j,j'$, the mask achieves the required.
\end{proof}

Using the previous result, we can show that a linear function $x\mapsto\inner{w^*,x}$ can be implemented by pruning a two layer network:

\begin{lemma}
\label{lem:linear_func}
Let $s \in [d]$, and
fix some $w^* \in [-\frac{1}{\sqrt{s}}, \frac{1}{\sqrt{s}}]^d$ with $\norm{w^*}_0 \le s$, and some $\epsilon, \delta > 0$.
Let $w^{(1)}, \dots, w^{(k)} \in \reals^d$ chosen randomly from $U([-1,1]^d)$, and $u \in [-1,1]^k$ chosen randomly from $U([-1,1]^k)$. Then, for $k \ge s \cdot \left \lceil \frac{16 s^{2}}{\epsilon^2} \log(\frac{2s}{\delta}) \right \rceil$, w.p at least $1-\delta$ there exists a binary mask $b^{(1)}, \dots, b^{(k)} \in \{0,1\}^d$, such that
$g(x) = \sum_{i=1}^k u_i \sigma(\inner{w^{(i)} \odot b^{(i)}, x})$ satisfies
$|g(x) - \inner{w^*,x}| \le \epsilon$, for $\norm{x}_\infty \le 1$. Furthermore, we have $\sum_{i} \norm{b^{(i)}}_0 \le 2s$ and $\max_{i} \norm{b^{(i)}}_0 \le 1$.
\end{lemma}

\begin{proof}
We assume $k = s \cdot \left \lceil \frac{16 s^{2}}{\epsilon^2} \log(\frac{2s}{\delta}) \right \rceil$ (otherwise, mask excessive neurons), and
let $k' := \frac{k}{s}$. With slight abuse of notation, we denote $w^{(i,j)} := w^{(j+k'i)}$, $u^{(i,j)} := u_{j+k'i}$ and $b^{(i,j)} := b^{(j+k'i)}$.
Let $I := \{i \in [d] ~:~ w^*_i \ne 0\}$. By the assumption on $w^*$ we have $|I| \le s$, and we assume w.l.o.g. that $I \subseteq [s]$.
Fix some $i \in [s]$, and denote $g_i(x) = \sum_j u^{(i,j)} \sigma(\inner{w^{(i,j)} \odot b^{(i,j)}, x})$.
Let $\epsilon' = \frac{\epsilon}{2s}$ and $\delta' = \frac{\delta}{s}$, then from Lemma \ref{lem:one_coord}, with probability at least $1-\delta'$ there exists a binary mask $b^{(i,1)}, \dots, b^{(i,k')} \in \{0,1\}^d$ with $\sum_j \norm{b^{(i,j)}}_0 \le 2$ such that $|g_i(x)-w^*_i x_i| \le 2\epsilon' = \frac{\epsilon}{s}$ for every $x \in \reals^d$ with $\norm{x}_\infty \le 1$. Now, using the union bound we get that with probability at least $1-\delta$, the above holds for all $i \in [s]$, and so:
\[
|g(x) - \inner{w^*,x}| = |\sum_{i \in [s]} g_i(x) - \sum_{i \in [s]} w^*_i x_i| \le 
\sum_{i \in [s]} |g_i(x) - w^*_i x_i| \le \epsilon
\]
Furthermore, we have $\sum_{i \in [s]} \sum_j \norm{b^{(i,j)}}_0 \le 2s$ and $\max_{i,j} \norm{b^{(i,j)}}_0 \le 1$, by the result of Lemma \ref{lem:one_coord}.
\end{proof}

Now, we can show that a network with a single neuron can be approximated by prunning a three-layer network:

\begin{lemma}
\label{lem:one_neuron}
Let $s \in [d]$, and
fix some $w^* \in [-\frac{1}{\sqrt{s}}, \frac{1}{\sqrt{s}}]^d$ with $\norm{w^*}_0 \le s$, some $v^*\in [-1,1]$ and some $\epsilon, \delta > 0$.
Let $w^{(1)}, \dots, w^{(k_1)} \in \reals^d$ chosen randomly from $U([-1, 1]^d)$, $u^{(1)}, \dots, u^{(k_2)} \in [-1,1]^{k_1}$ chosen randomly from $U([-1,1]^{k_1})$, and $v \in [-1,1]^{k_2}$ chosen randomly from $U([-1,1]^{k_2})$. Then, for $k_1 \ge s \cdot \left \lceil \frac{64 s^{2}}{\epsilon^2} \log(\frac{4s}{\delta}) \right \rceil$, $k_2 \ge \frac{2}{\epsilon} \log (\frac{2}{\delta})$, w.p at least $1-\delta$ there exists a binary mask $b^{(1)}, \dots, b^{(k_1)} \in \{0,1\}^d$, $\hat{b} \in \{ 0,1\}^{k_2}$, such that
$g(x) = \sum_{i=1}^{k_2} \hat{b}_i v_i \sigma(\sum_{j=1}^{k_1} u^{(i)}_j \sigma(\inner{w^{(j)} \odot b^{(j)}, x}))$ satisfies
$|g(x) - v^* \sigma(\inner{w^*,x})| \le \epsilon$, for $\norm{x}_2 \le 1$. Furthermore, we have $\sum_{j} \norm{b^{(j)}}_0 \le 2s$ and $\max_{j} \norm{b^{(j)}}_0 \le 1$.
\end{lemma}

\begin{proof}
Let $\epsilon' = \frac{\epsilon}{2}$, and note that for every $i \in [k_2]$ we have $\prob{\abs{v_i-v^*} \le \epsilon'} \ge \epsilon'$. Therefore, the probability that for some $i \in [k_2]$ it holds that $\abs{v_i - v^*} \le \epsilon'$ is at least $1-(1-\epsilon')^{k_2} \ge 1-e^{-k_2 \epsilon'} \ge 1-\frac{\delta}{2}$, where we use the fact that $k_2 \ge \frac{1}{\epsilon'} \log (\frac{2}{\delta})$.
Now, assume this holds for $i \in [k_2]$.
Let $\hat{b}_j =  \mathbbm{1}\{ j = i \}$, and so:
\[
g(x) = v_i \sigma(\sum_{j=1}^{k_1} u^{(i)}_j \sigma (\inner{w^{(j)} \circ b^{(j)}, x})
\]
Then, from \lemref{lem:linear_func}, with probability at least $1-\frac{\delta}{2}$ there exists $b^{(1)}, \dots, b^{(k_1)}$ s.t. for every $\norm{x}_\infty \le 1$:
\[
\abs{\sum_{j=1}^{k_1} u^{(i)}_j \sigma (\inner{w^{(j)} \circ b^{(j)}, x} - \inner{w^*, x})} \le \epsilon'
\]
And therefore, for every $\norm{x}_2 \le 1$:
\begin{align*}
&\abs{g(x) - v^* \sigma(\inner{w^*,x})} \\
&\le \abs{v_i} \abs{ \sigma(\sum_{j=1}^{k_1} u^{(i)}_j \sigma (\inner{w^{(j)} \circ b^{(j)}, x}) - \sigma(\inner{w^*, x}))}
+ \abs{v_i-v^*}\abs{\sigma(\inner{w^*,x})} \\
&\le \abs{v_i} \abs{\sum_{j=1}^{k_1} u^{(i)}_j \sigma (\inner{w^{(j)} \circ b^{(j)}, x} - \inner{w^*, x})} + \abs{v_i - v^*} \norm{w^*}\norm{x} \le 2 \epsilon' = \epsilon
\end{align*}
\end{proof}

Finally, we show that pruning a three-layer network can approximate a network with $n$ neurons, since it is only a sum of networks with 1 neuron, as analyzed in the previous lemma:

\begin{lemma}
\label{lem:relu_network}
Let $s \in [d]$, and
fix some $w^{(1)*}, \dots, w^{(n)*} \in [-1, 1]^d$ with $\norm{w^{(i)*}}_0 \le s$, $v^* \in [-1,1]^n$ and let $f(x) = \sum_{i=1}^n v^*_i \sigma(\inner{w^{(i)*}, x})$. Fix some $\epsilon, \delta > 0$.
Let $w^{(1)}, \dots, w^{(k_1)} \in \reals^d$ chosen randomly from $U([-1, 1]^d)$, $u^{(1)}, \dots, u^{(k_2)} \in [-1,1]^{k_1}$ chosen randomly from $U([-1,1]^{k_1})$, and $v \in [-1,1]^{k_2}$ chosen randomly from $U([-1,1]^{k_2})$. Then, for $k_1 \ge ns \cdot \left \lceil \frac{64 s^{2} n^2}{\epsilon^2} \log(\frac{4ns}{\delta}) \right \rceil$, $k_2 \ge \frac{2n}{\epsilon} \log (\frac{2n}{\delta})$, w.p at least $1-\delta$ there exists a binary mask $b^{(1)}, \dots, b^{(k_1)} \in \{0,1\}^d$, $\tilde{b}^{(1)}, \dots, \tilde{b}^{(k_2)} \in \{0,1\}^{k_1}$, $\hat{b} \in \{ 0,1\}^{k_2}$, such that
$g(x) = \sum_{i=1}^{k_2} \hat{b}_i v_i \sigma(\sum_{j=1}^{k_1} \tilde{b}^{(i)}_j u^{(i)}_j \sigma(\inner{w^{(j)} \odot b^{(j)}, x}))$ satisfies
$|g(x) - f(x)| \le \epsilon$, for $\norm{x}_2 \le 1$. Furthermore, we have $\sum_{j} \norm{b^{(j)}}_0 \le 2s$ and $\max_{j} \norm{b^{(j)}}_0 \le 1$.
\end{lemma}

\begin{proof}
Denote $k_1' = \frac{k_1}{n}, k_2' = \frac{k_2}{n}$ and assume $k_1', k_2' \in \mathbb{N}$ (otherwise mask exceeding neurons).
With slight abuse of notation, we denote $w^{(i,j)} := w^{(j+k_1'i)}$, $u^{(i,j)} := \left(u^{(j+ik_2')}_{ik_1'}, \dots, u^{(j+ik_2')}_{(i+1)k_1'} \right)$, $v^{(i,j)} := v_{j + ik_2'}$ and similarly $b^{(i,j)} := b^{(j+k_1'i)}$, $\tilde{b}^{(i,j)} = \left(\tilde{b}^{(j+ik_2')}_{ik_1'}, \dots, \tilde{b}^{(j+ik_2')}_{(i+1)k_1'} \right)$ and $\hat{b}^{(i,j)} = \hat{b}_{j+i k_2'}$.
Define for every $i \in [n]$:
\[
g_i(x) = \sum_j \hat{b}^{(i,j)} v^{(i,j)} \sigma(\sum_l \tilde{b}^{(i,j)}_l u^{(i,j)}_l \sigma(\inner{b^{(i,l)} \circ w^{(i,l)}, x}))
\]
Now, by setting $\tilde{b}^{(j+k_1'i)}_l = \mathbbm{1} \{ ik_1' \le l < (i+1)k_1' \}$ we get that $g(x) = \sum_{i=1}^n g_i(x)$. Now, from \lemref{lem:one_neuron} we get that with probability at least $1-\frac{\delta}{n}$ we have $\abs{g_i(x) - v_i^* \sigma(\inner{w^{(i)*}, x})} \le \frac{\epsilon}{n}$ for every $\norm{x}_2 \le 1$. Using the union bound, we get that with probability at least $1-\delta$, for $\norm{x}_2 \le 1$ we have $\abs{g(x)-f(x)} \le \sum_{i=1}^n \abs{g_i(x) - v_i^* \sigma(\inner{w^{(i)*}, x})} \le \epsilon$.
\end{proof}

\begin{proof} of Theorem \ref{thm:weight_pruning}.

From Lemma \ref{lem:relu_network} with $s = d$.
\end{proof}

In a similar fashion, we can prove a result for deep networks. We start by showing that a single layer can be approximated by pruning:
\begin{lemma}
\label{lem:one_layer}
Let $s \in [d]$, and
fix some $w^{(1)*}, \dots, w^{(n)*} \in [-\frac{1}{\sqrt{s}}, \frac{1}{\sqrt{s}}]^d$ with $\norm{w^{(i)*}}_0 \le s$ and let $F : \reals^d \to \reals^n$ such that
$F(x)_i = \sigma(\inner{w^{(i)*}, x})$. Fix some $\epsilon, \delta > 0$.
Let $w^{(1)}, \dots, w^{(k)} \in \reals^d$ chosen randomly from $U([-1, 1]^d)$ and $u^{(1)}, \dots, u^{(n)} \in [-1,1]^{k}$ chosen randomly from $U([-1,1]^{k})$.
Then, for $k \ge ns \cdot \left \lceil \frac{16 s^{2} n}{\epsilon^2} \log(\frac{2ns}{\delta}) \right \rceil$, w.p at least $1-\delta$ there exists a binary mask $b^{(1)}, \dots, b^{(k)} \in \{0,1\}^d$, $\tilde{b}^{(1)}, \dots, \tilde{b}^{(n)} \in \{0,1\}^{k_1}$, $\hat{b} \in \{ 0,1\}^{k}$, such that for $G : \reals^d \to \reals^n$ with 
$G(x)_i = \sigma(\sum_{j=1}^{k} \tilde{b}^{(i)}_j u^{(i)}_j \sigma(\inner{w^{(j)} \odot b^{(j)}, x}))$ we have
$\norm{G(x) - F(x)}_2 \le \epsilon$, for $\norm{x}_\infty \le 1$. Furthermore, we have $\sum_{j} \norm{b^{(j)}}_0 \le 2sn$ and $\sum_i \norm{\tilde{b}^{(i)}}_0 \le 2sn$.
\end{lemma}

\begin{proof}
Denote $k' = \frac{k}{n}$ and assume $k' \in \mathbb{N}$ (otherwise mask exceeding neurons).
With slight abuse of notation, we denote $w^{(i,j)} := w^{(j+k'i)}$, $b^{(i,j)} := b^{(j+k'i)}$ and we denote $\tilde{u}^{(i)} := \left(u^{(i)}_{ik'}, \dots, u^{(i)}_{(i+1)k'} \right)$.
Define for every $i \in [n]$:
\[
g_i(x) = \sum_j \tilde{u}^{(i)}_j \sigma(\inner{b^{(i,j)} \circ w^{(i,j)}, x})
\]
Now, by setting $\tilde{b}^{(j+k_1'i)}_l = \mathbbm{1} \{ ik_1' \le l < (i+1)k_1' \}$ we get that $G(x)_i = \sigma(g_i(x))$. Now, from \lemref{lem:linear_func} with $\epsilon' = \frac{\epsilon}{\sqrt{n}}$ and $\delta' = \frac{\delta}{n}$, since $k \ge s \cdot \left \lceil \frac{16 s^{2}}{(\epsilon')^2} \log(\frac{2s}{\delta'}) \right \rceil$ we get that with probability at least $1-\frac{\delta}{n}$ we have $\abs{g_i(x) - \inner{w^{(i)*}, x}} \le \frac{\epsilon}{\sqrt{n}}$ for every $\norm{x}_\infty \le 1$. Using the union bound, we get that with probability at least $1-\delta$, for $\norm{x}_\infty \le 1$ we have:
\[
\norm{G(x)-F(x)}_2^2 = \sum_i (\sigma(g_i(x)) - \sigma(\inner{w^{(i)*}, x}))^2 \le \sum_i (g_i(x) - \inner{w^{(i)*}, x})^2 \le \epsilon^2
\]
Notice that \lemref{lem:linear_func} also gives $\sum_j \norm{b^{(i,j)}}_0 \le 2s$ and so $\sum_{i=1}^n \sum_j \norm{b^{(i,j)}}_0 \le 2sn$. Since we can set $\tilde{b}^{(i)}_j = 0$ for every $i,j$ with $b^{(i,j)} = 0$, we get the same bound on $\sum_i \norm{\tilde{b}^{(i)}}_0$.
\end{proof}

Using the above, we can show that a deep network can be approximated by pruning. We show this result with the assumption that each neuron in the network has only $s$ non-zero weights. To get a similar result without this assumption, as is stated in \thmref{thm:deep_network}, we can simply choose $s$ to be its maximal value - either $d$ for the first layer of $n$ for intermediate layers.
\begin{theorem}
(formal statement of \thmref{thm:deep_network}, when $s = \max\{n,d\}$).
Let $s,n \in \mathbb{N}$, and
fix some $W^{(1)*}, \dots, W^{(l)*}$ such that $W^{(1)*} \in [-\frac{1}{\sqrt{s}}, \frac{1}{\sqrt{s}}]^{d \times n}$, $W^{(2)*}, \dots, W^{(l-1)*} \in [-\frac{1}{\sqrt{n}}, \frac{1}{\sqrt{n}}]^{n \times n}$ and $W^{(l)*} \in [-\frac{1}{\sqrt{n}}, \frac{1}{\sqrt{n}}]^{n \times 1}$. Assume that for every $i \in [l]$ we have $\norm{W^{(i)*}}_2 \le 1$ and $\max_j\norm{W^{(i)}_j}_0 \le s$.
Denote $F^{(i)}(x) = \sigma(W^{(i)*} x)$ for $i < l$ and $F^{(l)}(x) = W^{(l)*} x$, and let $F(x) := F^{(l)} \circ \dots \circ F^{(1)}(x)$. Fix some $\epsilon, \delta \in (0,1)$.
Let $W^{(1)}, \dots, W^{(l)}, U^{(1)}, \dots, U^{(l)}$ such that $W^{(1)}$ is chosen randomly from $U([-1, 1]^{d \times k})$, $W^{(2)}, \dots, W^{(l)}$ is chosen randomly from $U([-1, 1]^{n \times k})$, $U^{(1)}, \dots, U^{(l-1)}$ chosen from $U([-1,1]^{k \times n})$ and $U^{(l)}$ chosen from $U([-1,1]^k)$.
Then, for $k \ge ns \cdot \left \lceil \frac{64 s^{2}l^2 n}{\epsilon^2} \log(\frac{2nsl}{\delta}) \right \rceil$, w.p. at least $1-\delta$ there exist $B^{(i)}$ a binary mask for $W^{(i)}$ with matching dimensions, and $\tilde{B}^{(i)}$ a binary mask for $U^{(i)}$ with matching dimensions, s.t.:
\[
\abs{G(x) - F(x)} \le \epsilon~for~\norm{x}_2 \le 1
\]
Where we denote $G = G^{(l)} \circ \dots \circ G^{(1)}$, with $G^{(i)}(x) := \sigma(\tilde{B}^{(i)} \circ U^{(i)} \sigma( B^{(i)} \circ W^{(i)} x))$ for every $i < l$ and $G^{(l)}(x) := \tilde{B}^{(l)} \circ U^{(l)} \sigma(B^{(l)} \circ W^{(l)} x)$. Furthermore, we have $\norm{B^{(i)}}_0 \le 2sn$ and $\norm{\tilde{B}^{(i)}}_0 \le 2sn$.
\end{theorem}

\begin{proof}
Fix some $i < l$. From \ref{lem:one_layer}, with probability at least $1-\frac{\delta}{l}$ there exists a choice for $\tilde{B}^{(i)}, B^{(i)}$ such that for every $\norm{x}_\infty \le 1$ we have $\norm{F^{(i)}(x) - G^{(i)}(x)}_2 \le \frac{\epsilon}{2l}$. 
Note that we want to show that every layer is well approximated given the output of the previous layer, which can slightly deviate from the output of the original network. So, we need to relax the condition of \lemref{lem:one_layer} to  $\norm{x}_\infty \le 2$ in order to allow these small deviations from the target network.

Notice that if $\norm{x}_\infty \le 2$, from homogeneity of $G^{(i)}, F^{(i)}$ to positive scalars we get that:
\[
\norm{G^{(i)}(x)-F^{(i)}(x)}_2 = 2\norm{G^{(i)}(\frac{1}{2}x) -F^{(i)}(\frac{1}{2}x)}_2 \le \frac{\epsilon}{l}
\]
Similarly, from \lemref{lem:linear_func}, with probability at least $1-\frac{\delta}{l}$ it holds that $\abs{F^{(l)}(x) - G^{(l)}(x)} \le \frac{\epsilon}{l}$ for every $x$ with $\norm{x}_\infty \le 2$.
Assume that all the above holds, and using the union bound this happens with probability at least $1-\delta$.
Notice that for every $x$ we have $\norm{F^{(i)}(x)}_2 \le \norm{W^{(i)*} x}_2 \le \norm{W^{(i)*}}_2 \norm{x}_2 \le \norm{x}_2$, and so $\norm{F^{(i)} \circ \dots \circ F^{(1)}(x)}_2 \le \norm{F^{(i-1)} \circ \dots \circ F^{(1)}(x)}_2 \le \dots \le \norm{x}_2$.
Fix some $x$ with $\norm{x}_2 \le 1$ and denote $x^{(i)} = F^{(i)} \circ \dots \circ F^{(1)}(x)$ and $\hat{x}^{(i)} = G^{(i)} \circ \dots \circ G^{(1)}(x)$.
Now, we will show that $\norm{x^{(i)} - \hat{x}^{(i)}}_2 \le \frac{i\epsilon}{l}$ for every $i \le l$, by induction on $i$. The case $i = 0$ is trivial, and assume the above holds for $i-1$. Notice that in this case we have $\norm{\hat{x}^{(i-1)}}_\infty \le \norm{\hat{x}^{(i-1)}}_2 \le \norm{x^{(i-1)}}_2 + \norm{x^{(i-1)}- \hat{x}^{(i-1)}}_2 \le 2$. Therefore:
\begin{align*}
\norm{x^{(i)} - \hat{x}^{(i)} }_2
&= \norm{G^{(i)}(\hat{x}^{(i-1)}) - F^{(i)}(x^{(i-1)})}_2 \\
&\le \norm{G^{(i)}(\hat{x}^{(i-1)}) - F^{(i)}(\hat{x}^{(i-1)})}_2 + \norm{F^{(i)}(\hat{x}^{(i-1)}) - F^{(i)}(x^{(i-1)})}_2 \\
&\le \frac{\epsilon}{l} + \norm{ W^{(i)*} (\hat{x}^{(i-1)} - x^{(i-1)})}_2 \le \frac{\epsilon}{l} + \norm{W^{(i)*}}_2 \norm{\hat{x}^{(i-1)} - x^{(i-1)}}_2 \le \frac{i\epsilon}{l}
\end{align*}
From the above, we get that $\abs{F(x) - G(x)} = \norm{x^{(l)}-\hat{x}^{(l)}}_2 \le \epsilon$.
\end{proof}

\section{Proofs of Section \ref{sec:equivalence}}\label{appen:proofs from sec equivalence}

First we will need the following lemma, which intuitively shows a generalization bound over linear predictors, where each coordinate of each sample is pruned with equal probability and independently.
\begin{lemma}\label{lem:rademacher bound}
Let $k > 0$, and $v^{(1)},\dots,v^{(k)}\in[-1,1]^d$. Let $\hat{v}^{(j)}$ be Bernoulli random variables such that for each $j$, with probability $\epsilon$ we have $\hat{v}{(j)} = \frac{1}{\epsilon}v^{(j)}$, and with probability $1-\epsilon$ we have $\hat{v}^{(j)}=0$. Then we have w.p $> 1-\delta$ that:
\[
\sup_{z:\norm{z}\leq L}\abs{\frac{1}{k} \sum_{j=1}^{k} \inner{\hat{v}^{(j)}, z}  - \frac{1}{k} \sum_{j=1}^{k} \inner{{v}^{(j)}, z}} \leq \frac{L}{\epsilon \sqrt{k}}\left(3\sqrt{d} + \log\left(\frac{1}{\delta}\right)\right)
\]
\end{lemma}

\begin{proof}
Note that for each $j\in[k]$ we have that $\mathbb{E}\left[\hat{v}^{(j)}\right] = {v}^{(j)} $, thus for every vector $z\in\mathbb{R}^d$, also $\mathbb{E}\left[\frac{1}{k} \sum_{j=1}^{k} \inner{\hat{v}^{(j)}, z}\right] = \frac{1}{k} \sum_{j=1}^{k} \inner{{v}^{(j)}, z}$. Hence, using a standard argument about Rademacher complexity (see \cite{shalev2014understanding} Lemma 26.2) we have that:
\begin{align}\label{eq:eq:rademacher bound with sup}
    &\mathbb{E}_{\hat{v}^{(1)},\dots,\hat{v}^{(k)}}\left[\sup_{z:\norm{z}\leq L}  \abs{\frac{1}{k} \sum_{j=1}^{k} \inner{\hat{v}^{(j)}, z}  - \frac{1}{k} \sum_{j=1}^{k} \inner{{v}^{(j)}, z}}\right] \nonumber\\
    \leq& \frac{2}{k}  \mathbb{E}_{\hat{v}^{(1)},\dots,\hat{v}^{(k)}} \mathbb{E}_{\xi_1,\dots,\xi_k}\left[\sup_{z:\norm{z}\leq L}  \sum_{j=1}^k \xi_j \inner{\hat{v}^{(j)} - v^{(j)},z}\right]
\end{align}
where $\xi_1,\dots,\xi_k$ are standard Rademacher random variables. Set $\tilde{v}^{(j)} = \hat{v}^{(j)} - v^{(j)}$ ,using Cauchy-Schwartz we can bound \eqref{eq:eq:rademacher bound with sup} by:
\begin{align}\label{eq:rademacher bound without sup}
    \frac{2}{k}  \mathbb{E}_{\tilde{v}^{(1)},\dots,\tilde{v}^{(k)}} \mathbb{E}_{\xi_1,\dots,\xi_k}\left[\sup_{z:\norm{z}\leq L} \norm{z}\cdot \left\|\sum_{j=1}^k \xi_j \tilde{v}^{(j)}\right\|\right] \leq \frac{2L}{k}\mathbb{E}_{\tilde{v}^{(1)},\dots,\tilde{v}^{(k)}} \mathbb{E}_{\xi_1,\dots,\xi_k}\left[ \left\|\sum_{j=1}^k \xi_j \tilde{v}^{(j)}\right\|\right]~.
\end{align}
Next, we can use Jensen's inequality on \eqref{eq:rademacher bound without sup} to bound it
\begin{align*}
    &\frac{2L}{k}\mathbb{E}_{\tilde{v}^{(1)},\dots,\tilde{v}^{(k)}} \mathbb{E}_{\xi_1,\dots,\xi_k}\left[ \left\|\sum_{j=1}^k \xi_j \tilde{v}^{(j)}\right\|\right] \leq \frac{2L}{k}\sqrt{\mathbb{E}_{\tilde{v}^{(1)},\dots,\tilde{v}^{(k)}} \mathbb{E}_{\xi_1,\dots,\xi_k}\left[ \left\|\sum_{j=1}^k \xi_j \tilde{v}^{(j)}\right\|^2\right]} \\
    \leq &\frac{2L}{k} \sqrt{\mathbb{E}_{\tilde{v}^{(1)},\dots,\tilde{v}^{(k)}} \mathbb{E}_{\xi_1,\dots,\xi_k}\left[ \sum_{i=1}^k\sum_{j=1}^k \xi_i\xi_j \tilde{v}^{(i)^\top}\tilde{v}^{(j)}\right]} = \frac{2L}{k} \sqrt{\mathbb{E}_{\tilde{v}^{(1)},\dots,\tilde{v}^{(k)}}\left[ \sum_{j=1}^k \norm{\tilde{v}^{(j)}}^2\right]}~.
\end{align*}
Finally, using the fact that $\norm{\tilde{v}^{(j)}}^2 \leq \norm{\hat{v}^{(j)}}^2 + \norm{v^{(j)}}^2 \leq  \frac{1}{\epsilon^2}\norm{v^{(j)}}^2 + \norm{v^{(j)}}^2\leq \frac{2d}{\epsilon^2} $ we have that:
\begin{equation*}
    \frac{2L}{k} \sqrt{\mathbb{E}_{\tilde{v}^{(1)},\dots,\tilde{v}^{(k)}}\left[ \sum_{j=1}^k \norm{\tilde{v}^{(j)}}^2\right]} \leq \frac{3L\sqrt{d}}{\epsilon\sqrt{k}}
\end{equation*}
In order to prove the lemma we will use McDiarmid's inequality to get guarantees with high probability. Note that for every $l\in [k]$, by taking $\tilde{\hat{v}}^{(l)}$ instead of $\hat{v}^{(l)}$ we have for every $z$ with $\norm{z}\leq L$ that:
\[
\abs{\frac{1}{k}\sum_{j=1}^k\inner{\hat{v}^{(j)},z} - \frac{1}{k}\left(\sum_{j\neq l}\inner{\hat{v}^{(j)},z} - \inner{\tilde{\hat{v}}^{(l)},z}\right)} \leq \frac{1}{k}\abs{\inner{{\hat{v}}^{(l)},z} - \inner{\tilde{\hat{v}}^{(l)},z}} \leq \frac{L}{\epsilon k}
\]
By using Mcdiarmid's theorem we get
\begin{equation*}
    \mathbb{P}\left( \sup_{z:\norm{z}\leq L}\abs{\frac{1}{k} \sum_{j=1}^{k} \inner{\hat{v}^{(j)}, z}  - \frac{1}{k} \sum_{j=1}^{k} \inner{{v}^{(j)}, z}} \geq \frac{3L\sqrt{d}}{\epsilon k} + t \right) \leq \exp\left(-\frac{-2t^2\epsilon^2 k}{L^2}\right)~,
\end{equation*}
setting the r.h.s to $\delta$, and $t = \frac{\sqrt{\log\left(\frac{1}{\delta}\right)} L}{\epsilon \sqrt{k}}$ we have w.p $>1-\delta$ that:
\begin{equation*}
\sup_{z:\norm{z}\leq L}\abs{\frac{1}{k} \sum_{j=1}^{k} \inner{\hat{v}^{(j)}, z}  - \frac{1}{k} \sum_{j=1}^{k} \inner{{v}^{(j)}, z}} \leq \frac{L}{\epsilon \sqrt{k}}\left(3\sqrt{d} + \sqrt{\log\left(\frac{1}{\delta}\right)}\right)~.
\end{equation*}

\end{proof}

Next, we show the main argument, which states that by pruning a neuronds from a large enough $2$-layer neural network, it can approximate any other $2$-layer neural network for which the weights in the first layer are the same, and the weights in the second layer are bounded.
\begin{lemma}\label{lem:twice hoeffding}
Let $k_1\in\mathbb{N}$ and $\epsilon,\delta, M >0$ and assume that $\sigma$ is $L$-Lipschitz with $\sigma(0) \leq L$. Let $k_2 > \frac{256\log\left(\frac{2k_1}{\delta}\right)k_1^4L^4}{\epsilon^4}$, and for every $i\in [k_1], ~ j\in [k_2]$ initialize $w_i^{(j)} \sim \mathcal{D} $ for any distribution $\mathcal{D}$ with $\mathbb{P}\left(\norm{w_i}\leq 1\right)=1$ and $u_i^{(j)}\sim U([-1,1])$. Let $v^{(1)},\dots,v^{(k_2)}\in\mathbb{R}^{k_1}$ with $\norm{v^{(j)}}_\infty \leq M$ for every $j\in [k_2]$, and define $f^{(j)}(x) = \sum_{i=1}^{k_1}v_i^{(j)}\sigma\left(\inner{w_i^{(j)},x}\right)$. Then there exist $b^{(1)},\dots,b^{(k_2)}\in \{0,1\}^{k_1}$ such that for the functions $\tilde{g}^{(j)}(x)=\sum_{i=1}^{k_1}b_i^{(j)}\cdot u_i^{(j)}\sigma\left(\inner{w_i^{(j)},x}\right)$ w.p $> 1-\delta$ we have:
\begin{equation*}
    \sup_{x:\norm{x}\leq 1}\abs{\frac{c_1}{k_2}\sum_{j=1}^{k_2}\tilde{g}^{(j)}(x) - \frac{1}{k_2 M}\sum_{j=1}^{k_2}f^{(j)}(x)} \leq \epsilon
\end{equation*}
where $c_1=\frac{8k_1L}{\epsilon}$
\end{lemma}

\begin{proof}
Denote $\epsilon' = \frac{\epsilon}{4k_1L} $, and for $j\in[k_2]$ denote $\bar{v}^{(j)} = \frac{1}{M}v^{(j)}$, so we have $\norm{\bar{v}^{(j)}}_\infty \leq 1$. Let $b_i^{(j)} = \mathbbm{1} \left\{\left|u_i^{(j)} - \bar{v}_i^{(j)}\right| \le \epsilon'\right\}$, note that the $b_i^{(j)}$-s are i.i.d Bernoulli random variables with $\prob{b_i^{(j)} = 1} = \frac{\epsilon'}{2}$. 

Set the following vectors: $\hat{v}^{(j)} = \frac{2}{\epsilon'}\begin{pmatrix} b_1^{(j)} \bar{v}_1^{(j)} \\ \vdots \\ b_{k_1}^{(j)} \bar{v}_{k_1}^{(j)}\end{pmatrix}, ~ \hat{u}^{(j)} = \frac{2}{\epsilon' }\begin{pmatrix} b_1^{(j)} u_1^{(j)} \\ \vdots \\ b_{k_1}^{(j)} \bar{u}_{k_1}^{(j)}\end{pmatrix}$, and denote the function $z^{(j)}: \reals^d\rightarrow \reals^{k_1}$ with $z_i^{(j)}(x)= \sigma\left(\inner{w_i^{(j)}, x}\right)$. Now, the functions $f^{(j)}(x)$ can be written as $f^{(j)}(x)= \inner{v^{(j)},z^{(j)}(x)}$, we denote 
\begin{align*}
\tilde{g}(x) & = \sum_{j=1}^{k_2}\sum_{i=1}^{k_1} b_i^{(j)}{u}_i^{(j)} \sigma\left(\inner{w_i^{(j)},x}\right) = \sum_{j=1}^{k_2}\inner{b^{(j)}\odot {u}^{(j)},z^{(j)}(x)}\\
\hat{g}(x) &= \frac{2}{\epsilon'}\sum_{j=1}^{k_2}\sum_{i=1}^{k_1} b_i^{(j)}{u}_i^{(j)} \sigma\left(\inner{w_i^{(j)},x}\right) = \sum_{j=1}^{k_2}\inner{\hat{u}^{(j)},z^{(j)}(x)}.
\end{align*}

Our goal is to bound the following, when the supremum is taken over $\norm{x} \leq 1$:
\begin{align}\label{eq:distance of g hat to f j}
    & \sup_x\abs{\frac{c_1}{k_2}\tilde{g}(x) - \frac{1}{k_2M}\sum_{j=1}^{k_2} f^{(j)}(x)}= \sup_x\abs{\frac{1}{k_2}\hat{g}(x) - \frac{1}{k_2M}\sum_{j=1}^{k_2} f^{(j)}(x)} \nonumber\\
    &=  \sup_x\abs{\frac{1}{k_2} \sum_{j=1}^{k_2} \inner{\hat{u}^{(j)}, z^{(j)}(x)}  - \frac{1}{k_2}\sum_{j=1}^{k_2} \inner{\bar{v}^{(j)}, z^{(j)}(x)}} \nonumber\\
    &\le \sup_x\abs{\frac{1}{k_2} \sum_{j=1}^{k_2} \inner{\hat{u}^{(j)}, z^{(j)}(x)}  - \frac{1}{k_2} \sum_{j=1}^{k_2} \inner{\hat{v}^{(j)}, z^{(j)}(x)}} + 
    \sup_x\abs{\frac{1}{k_2} \sum_{j=1}^{k_2} \inner{\hat{v}^{(j)}, z^{(j)}(x)}  - \frac{1}{k_2} \sum_{j=1}^{k_2} \inner{\bar{v}^{(j)}, z^{(j)}(x)}} \nonumber\\
\end{align}
where $c_1 = \frac{2}{\epsilon'} = \frac{8k_1 L}{\epsilon}$. We will now bound each expression in \eqref{eq:distance of g hat to f j} with high probability. For the first expression, we first bound:
\begin{align*}
    \sup_x\abs{\frac{1}{k_2} \sum_{j=1}^{k_2} \inner{\hat{u}^{(j)}, z^{(j)}(x)}  - \frac{1}{k_2} \sum_{j=1}^{k_2} \inner{\hat{v}^{(j)}, z^{(j)}(x)}} &= \sup_x\abs{\frac{1}{k_2} \sum_{j=1}^{k_2} \inner{\hat{u}^{(j)} - \hat{v}^{(j)}, z^{(j)}(x)}} \\
    & \leq \frac{1}{k_2} \sum_{j=1}^{k_2}\sup_x\abs{\inner{\hat{u}^{(j)} - \hat{v}^{(j)}, z^{(j)}(x)}}~.
\end{align*}

Fix $i\in [k_1]$ and set $X_i^{(j)} := \sup_x\abs{\left(\hat{u}^{(j)}_i - \hat{v}_i^{(j)}\right)\cdot z_i^{(j)}(x)}$ and note that for every $x$ with $\norm{x} \leq 1$ we have that $\sup_x \abs{z_i^{(j)}(x)} \leq 2L$. For the random variables $X_i^{(j)}$ we get:
\begin{itemize}
    \item $X_i^{(j)} \leq \abs{\hat{u}^{(j)}_i - \hat{v}_i^{(j)}} \cdot \sup_x\abs{z_i^{(j)}(x)} \leq 4L $
    \item $\mathbb{E}\left[X_i^{(j)}\right] \leq 2\epsilon' L$
\end{itemize}
We now use Hoeffding's inequality to get that:
\[
\mathbb{P}\left( \frac{1}{k_2}\sum_{j=1}^{k_2}X_i^{(j)} \geq 2\epsilon' L + t \right) \leq \exp\left(- \frac{t^2k_2}{8L^2}\right).
\]
Replacing the r.h.s with $\delta_1$ and setting $t= \epsilon' L$, we get that if $k_2 \geq \frac{8\log\left(\frac{1}{\delta_1}\right)}{\epsilon'^2}$ then w.p $1-\delta_1$: 
\begin{equation*}
    \frac{1}{k_2}\sup_x\abs{\left(\hat{u}^{(j)}_i - \hat{v}_i^{(j)}\right)\cdot z_i^{(j)}(x)} \leq 3\epsilon' L.
\end{equation*}
Setting $\delta_1 = \frac{\delta}{2k_1}$, and applying union bound for $i=1,\dots,k_1$ we get that w.p $> 1- \frac{\delta}{2}$ we have:
\begin{equation}\label{eq:u hat v hat sum}
    \frac{1}{k_2} \sum_{j=1}^{k_2}\sup_x\abs{\inner{\hat{u}^{(j)} - \hat{v}^{(j)}, z^{(j)}(x)}} \leq 3 k_1 \epsilon' L.
\end{equation}

For the second expression in \eqref{eq:distance of g hat to f j} we first note that for all $j\in[k_2]$ we have $\max_{x:\norm{x}\leq 1} \norm{z^{(j)}(x)} \leq 2L\sqrt{k_1}$. Hence we can bound the second expression 
\begin{align*}
     &\sup_x\abs{\frac{1}{k_2} \sum_{j=1}^{k_2} \inner{\hat{v}^{(j)}, z^{(j)}(x)}  - \frac{1}{k_2} \sum_{j=1}^{k_2} \inner{\bar{v}^{(j)}, z^{(j)}(x)}}  \\
     \leq & \sum_{z\in\mathbb{R}^{k_1}:\norm{z}\leq 2L\sqrt{k_1}} \abs{\frac{1}{k_2} \sum_{j=1}^{k_2} \inner{\hat{v}^{(j)}, z}  - \frac{1}{k_2} \sum_{j=1}^{k_2} \inner{\bar{v}^{(j)}, z}}~.
\end{align*}
Using \lemref{lem:rademacher bound} on the above term, w.p $> 1-\frac{\delta}{2}$ we have that:
\begin{equation}\label{eq:second term bound}
    \sum_{z\in\mathbb{R}^{k_1}:\norm{z}\leq 2L\sqrt{k_1}} \abs{\frac{1}{k_2} \sum_{j=1}^{k_2} \inner{\hat{v}^{(j)}, z}  - \frac{1}{k_2} \sum_{j=1}^{k_2} \inner{\bar{v}^{(j)}, z}} \leq \frac{2L\sqrt{k_1}}{\epsilon'\sqrt{k_2}}\left( 3\sqrt{k_1} + \sqrt{\log\left(\frac{2}{\delta}\right)}\right)
\end{equation}
Combining \eqref{eq:u hat v hat sum} with \eqref{eq:second term bound}, applying union bound and taking $k_2 \geq \frac{256L^4k_1^4\log\left(\frac{2}{\delta}\right)}{\epsilon^4}$, we can now use the bound in \eqref{eq:distance of g hat to f j} to get w.p $ > 1-\delta$:
\begin{equation*}
    \sup_x\abs{\frac{1}{k_2}\hat{g}(x) - \frac{1}{k_2M}\sum_{j=1}^{k_2} f^{(j)}(x)} \leq  \epsilon~.
\end{equation*}
\end{proof}

We are now ready to prove the main theorem:

\begin{proof}[Proof of \thmref{thm:equivalence nn rf}]
Set $m = \frac{256\log\left(\frac{2n}{\delta}\right)C^4n^4L^4}{\epsilon^4} \cdot \frac{\log\left(\frac{1}{\delta}\right)}{2\delta^3}$ and initialize a $2$-layer neural network with width $k :=m\cdot n$ and initialization as described in the theorem, denote  $g(x) = \sum_{j=1}^m \sum_{i=1}^n u_i^{(j)}\sigma(\inner{w_i^{(j)},x})$ as this network. By the assumption of the theorem, for each $j\in[m]$ w.p $> 1-\delta$ there exists a vector $v^{(j)}$ with $\norm{v^{(j)}}_\infty \leq C$ such that the function $f^{(j)}(x) = \sum_{i=1}^n v_i^{(j)}\sigma(\inner{w_i^{(j)},x})$ satisfy that $L_D\left(f^{(j)}\right) \leq \epsilon$. Let $Z_j$ be the random variable such that $Z_j = 0$ if there exists a vector $v^{(j)}$ that satisfies the above, and $Z_j=1$ otherwise. the random variables $Z_j$ are i.i.d since we initialize each $w_i^{(j)}$ i.i.d, and $\mathbb{P}(Z_j=1) = \delta$, $\mathbb{E}[Z_j] = \delta$. Denote $Z = \sum_{j=1}^m Z_j$, then $\mathbb{E}[Z] = m\delta$. We use Hoeffding's inequality on $Z$ to get that:
\[
\mathbb{P}\left( \frac{1}{m}Z \geq \delta + t\right) \leq \exp(-2mt^2)~.
\]
Replacing the r.h.s with $\delta$ and setting $t=\delta$ we get that if $m > \frac{\log\left(\frac{1}{\delta}\right)}{2\delta^2}$ then w.p $>1-\delta$ we have that $Z \leq 2\delta$. In particular, there are at least $m_0 =  \frac{256\log\left(\frac{2n}{\delta}\right)C^4n^4L^4}{\epsilon^4}$ indices (denote them w.l.o.g $j=1,\dots,m_0$) such that for every $j\in[m_0]$ there exists a vector $v^{(j)}$ with $\norm{v^{(j)}}_\infty \leq C$ such that the function $f^{(j)}(x) = \sum_{i=1}^n v_i^{(j)}\sigma(\inner{w_i^{(j)},x})$ satisfy that $L_D\left(f^{(j)}\right) \leq \epsilon$.

We now use \lemref{lem:twice hoeffding} with $\delta, \frac{\epsilon}{C}$ and $v^{(1)},\dots, v^{(m_0)}$ to get that w.p $ > 1-\delta$ that there exists a neuron-subnetwork $\tilde{g}(x)$ and constant $c' >0$ such that:
\begin{equation}\label{eq:bound with c' and tilde g}
\sup_{x:\norm{x}\leq 1}\abs{c'\tilde{g}(x) - \frac{1}{m_0 C}\sum_{j=1}^{m_0}f^{(j)}(x)} \leq \frac{\epsilon}{C}
\end{equation}
Set $c= C\cdot c'$, the loss of $c\tilde{g}(x)$ can be bounded by:
\begin{align}\label{eq:three summands to bound}
    L_D(c\tilde{g}) =& \mathbb{E}_{(x,y)\sim D}\left[(c\tilde{g}(x) - y)^2\right] \nonumber\\
    & \leq  2\mathbb{E}_{(x,y)\sim D}\left[\left(c\tilde{g}(x) -\frac{1}{m_0}\sum_{j=1}^{m_0}f^{(j)}(x)  \right)^2 \right] + 2\mathbb{E}_{(x,y)\sim D}\left[\left(\frac{1}{m_0}\sum_{j=1}^{m_0}f^{(j)}(x) -y \right)^2 \right]
\end{align}

We will bound each term of the above expression. Using \eqref{eq:bound with c' and tilde g} we have:
\begin{align}\label{eq:first summand out of three}
     \mathbb{E}_{(x,y)\sim D}\left[\left(c\tilde{g}(x) -\frac{1}{m}\sum_{j=1}^m f^{(j)}(x)  \right)^2 \right] &\leq \sup_{x:\norm{x}\leq 1}\left( c\tilde{g}(x) -\frac{1}{m}\sum_{j=1}^m f^{(j)}(x)\right)^2 \nonumber\\
    &\leq C\cdot \sup_{x:\norm{x}\leq 1}\left( c'\tilde{g}(x) -\frac{1}{mC}\sum_{j=1}^m f^{(j)}(x)\right)^2 \leq C\cdot \frac{\epsilon}{C} = \epsilon
\end{align}
For the second term in \eqref{eq:three summands to bound} we have that:
\begin{align}\label{eq:second summand out of three}
    \mathbb{E}_{(x,y)\sim D}\left[\left(\frac{1}{m}\sum_{j=1}^m f^{(j)}(x) -y \right)^2 \right] &\leq \frac{1}{m} \sum_{j=1}^m\mathbb{E}_{(x,y)\sim D}\left[ \left(f^{(j)}(x) - y\right)^2\right] \nonumber\\
    & \leq \frac{1}{m}\sum_{j=1}^m L_D\left(f^{(j)}\right) \leq \epsilon
\end{align}
re-scaling $\epsilon$ finishes the proof.
\end{proof}

\section{Proofs of section \ref{sec:pruning neurons}}\label{apen:proofs of pruning neurons}

We first show that a finite dataset, under mild assumptions on the data, can be approximated using a random features model. The proof of the following lemma is exactly the same as the proof of Lemma 3.1 in \cite{du2018gradient2}.

\begin{lemma}
Let $\delta > 0$,  $x_1,\dots,x_m\in\mathbb{R}^d$, and let $H$ be the $m\times m$ matrix with:
\[
H_{i,j} = \mathbb{E}_w\left[ \sigma(\inner{w,x_i})\sigma(\inner{w,x_j})\right]
\]
Assume that $\lambda_{\min}(H) = \lambda > 0$, then for $k > \frac{64m^2\log^2\left(\frac{m}{\delta}\right)}{\lambda^2}$, w.p $> 1-\delta$ over sampling of $w_1,\dots,w_k$ we have that $\lambda_{\min}(\tilde{H}) \geq \frac{3}{4}\lambda$ where:
\[
\tilde{H}_{i,j} = \sum_{l=1}^k \sigma(\inner{w_l,x_i})\sigma(\inner{w_l,x_j})
\]
\end{lemma}

Using the lemma above, and under the assumptions made on the data, w.h.p a two-layer network of size $\tilde{O}\left(\frac{m^2}{\lambda^2}\right)$ can overfit the data:

\begin{proposition}\label{prop:random features approximate data}
Let $\delta >0$, $x_1,\dots,x_m\in\mathbb{R}^d$ and $y_1,\dots,y_m\in\{\pm1\}$. Assume that $\lambda_{\min}(H) = \lambda >0$, and $\sigma$ is $L$-Lipschitz then for $k > \frac{64m^2\log^2\left(\frac{m}{\delta}\right)}{\lambda^2}$ w.p $1-\delta$ over sampling of $w_1,\dots,w_k$ there is $u\in\mathbb{R}^k$ with $\|u\|_\infty \leq \frac{4Lm}{3\lambda}$ such that for every $j=1,\dots,m$ we have $\sum_{i=1}^k u_i\sigma(\inner{w_i,x_j}) = y_j$
\end{proposition}

\begin{proof}
Set $X$ to be the $k\times m$ matrix defined by $X_{i,j}= \sigma(\inner{w_i,x_j})$. By our assumption and the choice of $k$, w.p $ > 1-\delta$ we have that $\tilde{H} = X^\top X$ is invertible, and has a minimal eigenvalue of at least $\frac{3}{4}\lambda$. Define $u = y(X^\top X)^{-1}X^\top$, it is easy to see that $uX = y$, furthermore:
\begin{align*}
\|u\|_\infty &= \|y(X^\top X)^{-1}X^\top\|_\infty \leq \frac{4}{3\lambda}\|Xy\|_\infty \\
&\leq \frac{4}{3\lambda} m \max_{w,x}\sigma(\inner{w,x}) \leq \frac{4Lm}{3\lambda}
\end{align*}

\end{proof}

For the second variation of \thmref{thm:neruon subnetwrok two variations} we consider functions from the class of functions $\mathcal{F}_C$. Here we use Theorem 3.3 from \cite{yehudai2019power}:
\begin{theorem}
\label{thm:random_features}

Let $f(x)= c_d \int_{w\in \left[\frac{-1}{\sqrt{d}},\frac{1}{\sqrt{d}}\right]^d}g(w)\sigma(\inner{w,x})dw$ where $\sigma:\mathbb{R}\rightarrow \mathbb{R}$ is $L$-Lipschitz on $[-1,1]$ with $\sigma(0) \leq L$, and $c_d = \left( \frac{\sqrt{d}}{2} \right)^d$ a normalization term. Assume that $max_{\|w\|\leq 1} |g(w)| \leq C$ for a constant $C$. Then for every $\delta >0$ if $w_1,\dots,w_k$ are drawn i.i.d from the uniform distribution on  $\left[\frac{-1}{\sqrt{d}},\frac{1}{\sqrt{d}}\right]^d$ , w.p $> 1-\delta$ there is a function of the form
  \[ \hat{f}(x) = \sum_{i=1}^{k}u_i \sigma(\inner{w_i,x}) \]
  where $|u_i| \leq \frac{C}{k}$ for every $1\leq i \leq k$, such that:
  \[ \sup_x \left|\hat{f}(x) - f(x) \right| \leq \frac{LC}{\sqrt{k}}\left(4 + \sqrt{2 \log\left(\frac{1}{\delta}\right)}\right) \]
\end{theorem}

To prove the main theorem, we use the same argument as in the proof of \thmref{thm:equivalence nn rf}, that pruning neurons can approximate random features models. Here the size of the target random features model depends on the complexity of the target (either a finite dataset or RKHS function).

\begin{proof}[Proof of \thmref{thm:neruon subnetwrok two variations}]
Although the proof for the two variations of the theorem are similar, for clarity and ease of notations we will prove them separately. 
\begin{enumerate}
    \item (Finite dataset) Let $\epsilon,\delta >0$. Fix $\delta_1 = \frac{\delta}{2k_2}$, and fix some $j \in [k_2]$.
    Take $k_1 \geq \frac{64m^2\log^2\left(\frac{m}{\delta_1}\right)}{\lambda^2}$, from \propref{prop:random features approximate data}  w.p $> 1 - \delta_1$ we get the following:
    There exists some $v^{(j)} \in \reals^{k_1}$ with $\norm{v^{(j)}}_\infty \le \frac{4Lm}{3\lambda}$  such that for the function $f^{(j)}(x) := \sum_{i=1}^{k_1} v_i^{(j)}\sigma\left(\inner{w^{(j)}_i,x}\right)$, and for every $l=1,\dots,m$, we have $f^{(j)}(x_l) = y_l$. Using union bound over all choices of $j$, we get that w.p $> 1-\frac{\delta}{2}$ the above hold for every $j \in [k_2]$.
    
    Denote $M :=\frac{4Lm}{3\lambda}$, $\epsilon' = \frac{\epsilon}{M} = \frac{3\lambda \epsilon}{4 L m}$ and let $k_2 > \frac{810L^8 m^4 k_1^4 \log\left(\frac{2k_1}{\delta}\right)}{\lambda^4 \epsilon^4}$. Using \lemref{lem:twice hoeffding} with $v^{(1)},\dots,v^{(k_2)}$ and $\epsilon'$ we have that there exist $b^{(1)},\dots,b^{(k_2)}$ such that for the functions $\tilde{g}^{(j)}(x) = \sum_{i=1}^{k_1}b_i^{(j)}\cdot u_i^{(j)} \sigma\left(\inner{w_i^{(j)},x}\right)$ we get:
    \begin{equation}\label{eq:finite dataset bound from the lemma}
        \sup_{x:\norm{x}\leq 1}\abs{\frac{c_1}{k_2}\sum_{j=1}^{k_2}\tilde{g}^{(j)}(x) - \frac{1}{k_2 M}\sum_{j=1}^{k_2}f^{(j)}(x)} \leq \epsilon'
    \end{equation}
    where $c_1=\frac{8k_1L}{\epsilon}$. Denote $\tilde{g}(x) = \sum_{j=1}^{k_2} g^{(j)}(x)$ and set $c = \frac{c_1 M}{k_2} = \frac{32k_1 L m}{3\lambda \epsilon k_2}$. Using \eqref{eq:finite dataset bound from the lemma} we have that for every $l=1,\dots,m$:
    \begin{align*}
        \abs{c\tilde{g}(x_l) - y_l} = \abs{\frac{c_1 M}{k_2}\tilde{g}(x_l) - \frac{1}{k_2}\sum_{j=1}^{k_2} f^{(j)}(x_l)} \leq M\epsilon' \leq \epsilon
    \end{align*}
    
    \item  Let $\epsilon,\delta >0$. Fix $\delta_1 = \frac{\delta}{2k_2}$, and fix some $j \in [k_2]$.
    Take $k_1 \geq \frac{128L^2C^2\log^2\left(\frac{m}{\delta_1}\right)}{\epsilon^2}$, from \thmref{thm:random_features} w.p $> 1 - \delta_1$ we get the following:
    There exists some $v^{(j)} \in \reals^{k_1}$ with $\norm{v^{(j)}}_\infty \le \frac{C}{k_1} \leq 1$  such that for the function $f^{(j)}(x) := \sum_{i=1}^{k_1} v_i^{(j)}\sigma\left(\inner{w^{(j)}_i,x}\right)$ we have $\sup_{x:\norm{x}\leq 1} \abs{f^{(j)}(x) - f(x)} \leq \frac{\epsilon}{2}$. Using union bound over all choices of $j$, we get that w.p $> 1-\frac{\delta}{2}$ the above hold for every $j \in [k_2]$.
    
    Let $k_2 > \frac{4010L^4 k_1^4 \log\left(\frac{2k_1}{\delta}\right)}{\epsilon^4}$, using \lemref{lem:twice hoeffding} with $v^{(1)},\dots,v^{(k_2)}$ and $\frac{\epsilon}{2}$ we have that there exist $b^{(1)},\dots,b^{(k_2)}$ such that for the functions $\tilde{g}^{(j)}(x) = \sum_{i=1}^{k_1}b_i^{(j)}\cdot u_i^{(j)} \sigma\left(\inner{w_i^{(j)},x}\right)$ we get:
    \begin{equation}\label{eq:RKHS bound from the lemma}
        \sup_{x:\norm{x}\leq 1}\abs{\frac{c_1}{k_2}\sum_{j=1}^{k_2}\tilde{g}^{(j)}(x) - \frac{1}{k_2 M}\sum_{j=1}^{k_2}f^{(j)}(x)} \leq \frac{\epsilon}{2}
    \end{equation}
    where $c_1=\frac{8k_1L}{\epsilon}$. Denote $\tilde{g}(x) = \sum_{j=1}^{k_2} g^{(j)}(x)$ and set $c = \frac{c_1}{k_2} = \frac{8 k_1 L}{\epsilon k_2}$. Using \eqref{eq:RKHS bound from the lemma} we have that:
    \begin{align*}
        &\sup_{x:\norm{x}\leq 1}\abs{c\tilde{g}(x) - f(x)} \\
        &\leq 
        \sup_{x:\norm{x}\leq 1}\abs{\frac{c_1}{k_2}\tilde{g}(x) - \frac{1}{k_2}\sum_{j=1}^{k_2} f^{(j)}(x)} + \sup_{x:\norm{x}\leq 1}\abs{\frac{1}{k_2}\sum_{j=1}^{k_2} f^{(j)}(x) - f(x)} \leq \frac{\epsilon}{2} + \frac{\epsilon}{2}= \epsilon
    \end{align*}
    
    \end{enumerate}
\end{proof}

\end{document}